%% file: main_arXiv.tex
\documentclass[lettersize,journal]{IEEEtran}
\usepackage{amsmath,amsfonts}
\usepackage{array}
\usepackage[caption=false,font=normalsize,labelfont=sf,textfont=sf]{subfig}
\usepackage{textcomp}
\usepackage{stfloats}
\usepackage{url}
\usepackage{verbatim}
\usepackage{graphicx}
\usepackage{cite}
\usepackage{amsmath,amssymb,amsfonts}
\usepackage{algorithmic}
\usepackage{xcolor}

\usepackage{amsmath, bm} 
\usepackage{amsthm}
\usepackage[ruled,norelsize,linesnumbered]{algorithm2e}
\makeatletter
\newcommand{\removelatexerror}{\let\@latex@error\@gobble}
\makeatother

\DeclareMathOperator*{\Argmin}{argmin}
\DeclareMathOperator*{\Max}{max}

\begin{document}
	
	\title{Task-Parameterized Imitation Learning\\with Time-Sensitive Constraints}
	
	\author{
		Julian Richter$^{1,2}$,
		João Oliveira$^{3}$,
		Christian Scheurer$^{1}$,
		Jochen Steil$^{2}$ and
		Niels Dehio$^{1}$%
		\thanks{$^{1}$Technology and Innovation Center (TIC), KUKA Deutschland GmbH, Germany, e-mail: {\tt\small Julian.Richter@kuka.com}}%
		\thanks{$^{2}$Institute of Robotics and Process Control (IRP), Technical University Braunschweig, Germany}%
		\thanks{$^{3}$IDMEC, Instituto Superior Técnico, Portugal}
		\thanks{This work has been submitted to the IEEE for possible publication. Copyright may be transferred without notice, after which this version may no longer be accessible.}%
	}%
	
	\maketitle
	
	\input{custom_commands}
	
	\input{abstract}
	
	\input{introduction}
	\input{background}
	\input{fundamentals}
	\input{enforcing_constraints}
	\input{robot_experiments}
	\input{conclusion}

	\section*{Acknowledgments}
	This work was partly supported by KUKA Deutschland GmbH, the state of Bavaria through the OPERA project DIK-2107-0004/DIK0374/01 and FCT (2020.07160.BD), through IDMEC, under LAETA, project UIDB/50022/2020. 

	\input{appendix}

	\bibliographystyle{ieeetr}
	\bibliography{biblio}

	\vfill
	
\end{document}

%% file: custom_commands.tex

\newtheorem{theorem}{Theorem}[section]
\newtheorem{corollary}{Corollary}[theorem]
\newtheorem{lemma}[theorem]{Lemma}

\newcommand{\IndexGaussian}[1]{{
		\lambda_{#1}
}}
\newcommand{\IndexInput}[1]{{
		\eta_{#1}
}}
\newcommand{\weighting}{{
		\bm{\omega}
}}
\newcommand{\scaling}[1]{{
		\gamma_{#1}
}}

\newcommand{\denominator}{\alpha}
\newcommand{\expArg}{\beta}

\newcommand{\todo}[1]{{\color{red}{#1}}}

\newcommand{\red}[1]{{\color{red}{#1}}}
\newcommand{\green}[1]{{\color{green}{#1}}}
\newcommand{\blue}[1]{{\color{blue}{#1}}}
\newcommand{\orange}[1]{{\color{orange}{#1}}}
\newcommand{\purple}[1]{{\color{purple}{#1}}}
\newcommand{\yellow}[1]{{\color{yellow}{#1}}}
\newcommand{\algComment}[1]{{\blue{// #1}}}

\newcommand{\AlgorithmInput}{\textbf{Input:}}
\newcommand{\AlgorithmOutput}{\textbf{Output:}}

\newcommand{\inR}{\in \mathbb{R}}
\newcommand{\origin}[1]{{
		\bm{I}_{#1}
}}
\newcommand{\zero}{\bm{0}}
\newcommand{\norm}[1]{{
		\Vert #1 \Vert
}}

\newcommand{\DataSet}[1]{{
		\bm{\Upsilon}_{#1}
}}
\newcommand{\DataAll}[2]{{
		{}^{#2}_{}\bm{\xi}_{#1}
}}
\newcommand{\DataIn}{t}
\newcommand{\DataOut}{\mathbf{x}}

\newcommand{\prob}{\mathcal{P}}
\newcommand{\likelihood}{\mathcal{L}}
\newcommand{\ModelParam}[1]{{
		\bm{\Omega}_{#1}
}}
\newcommand{\Responsibility}[2]{{
		r_{#1,#2}
}}
\newcommand{\latent}[1]{{
		z_{#1}
}}

\newcommand{\gauss}{\mathcal{N}}
\newcommand{\Prior}[1]{{
		\pi_{#1}
}}
\newcommand{\Mean}[3]{{
		{}^{#3}_{}\bm{\mu}_{#1}^{#2}
}}
\newcommand{\Covariance}[3]{{
		{}^{#3}_{}\bm{\Sigma}_{#1}^{#2}
}}

\newcommand{\ExpWeight}[3]{{
		{}^{#3}_{}\hat{h}_{#1}^{#2}
}}
\newcommand{\ExpMean}[3]{{
		{}^{#3}_{}\hat{\bm{\mu}}_{#1}^{#2}
}}
\newcommand{\ExpCovariance}[3]{{
		{}^{#3}_{}\hat{\bm{\Sigma}}_{#1}^{#2}
}}
\newcommand{\ExpWeightLocal}[3]{{
		{}^{#3}_{}\tilde{h}_{#1}^{#2}
}}
\newcommand{\ExpMeanLocal}[3]{{
		{}^{#3}_{}\tilde{\bm{\mu}}_{#1}^{#2}
}}
\newcommand{\ExpCovarianceLocal}[3]{{
		{}^{#3}_{}\tilde{\bm{\Sigma}}_{#1}^{#2}
}}

\newcommand{\MeanGlobal}[3]{{
		{}^{#3}_{}\tilde{\bm{\mu}}_{#1}^{#2}
}}
\newcommand{\CovarianceGlobal}[3]{{
		{}^{#3}_{}\tilde{\bm{\Sigma}}_{#1}^{#2}
}}
\newcommand{\CovariancePT}[3]{{
		{}^{#3}_{}\bm{E}_{#1}^{#2}
}}

\newcommand{\Manifold}{{
		\mathcal{M}
}}
\newcommand{\Tangent}[1]{{
		\mathcal{T}_{#1}
}}
\newcommand{\logmap}[2]{{
		\mathbf{Log}_{#2}(#1)
}}
\newcommand{\expmap}[2]{{
		\mathbf{Exp}_{#2}(#1)
}}
\newcommand{\parallelTransport}[2]{{
		\Gamma_{#1 \mapsto #2}
}}
\newcommand{\DataManifold}[2]{{
		{}^{#2}_{}\bm{x}_{#1}
}}
\newcommand{\DataManifoldTmp}[2]{{
		{}^{#2}_{}\bm{y}_{#1}
}}
\newcommand{\DataTangent}[3]{{
		{}^{#3}_{}\bm{u}_{#1}^{#2}
}}

\newcommand{\TaskA}[1]{{
		\mathbf{A}_{#1}
}}
\newcommand{\TaskB}[1]{{
		\mathbf{b}_{#1}
}}
\newcommand{\TaskQ}[1]{{
		\mathbf{q}_{#1}
}}

%% file: abstract.tex
\begin{abstract}
Programming a robot manipulator should be as intuitive as possible.
To achieve that, the paradigm of teaching motion skills by providing few demonstrations has become widely popular in recent years.
Probabilistic versions thereof take into account the uncertainty given by the distribution of the training data.
However, precise execution of start-, via-, and end-poses at given times can not always be guaranteed.
This limits the technology transfer to industrial application.
To address this problem, 
we propose a novel constrained formulation of the Expectation Maximization algorithm 
for learning Gaussian Mixture Models (GMM) on Riemannian Manifolds.
Our approach applies to probabilistic imitation learning
and extends also to the well-established TP-GMM framework with Task-Parameterization.
It allows to prescribe end-effector poses at defined execution times,
for instance for precise pick \& place scenarios.
The probabilistic approach is compared with state-of-the-art learning-from-demonstration methods using the KUKA LBR iiwa robot.
The reader is encouraged to watch the accompanying video
available at \url{https://youtu.be/JMI1YxtN9C0}.
\end{abstract}

%% file: introduction.tex
\section{Introduction}
\label{sec:introduction}

Robot manufacturers seek to provide easy-to-use programming interfaces,
allowing non-experts to teach robot skills intuitively.
One common aim is to empower factory-floor workers in transmitting their tool-handling skills to the robot.
This would enable small and medium-sized companies to automate repetitive processes with little effort, 
resulting in improved productivity and potential economic growth.
Underlying these efforts is research in \emph{Learning from Demonstration}~(LfD)~\cite{
	Billard_RecentAdvancesInRobotLearningFromDemonstration_2020,
	Saveriano_LearningStableRoboticSkillsOnRiemannianManifolds_2022
},
also referred to as
\emph{Programming by Demonstration}~\cite{
	Khansari_SEDS_2011,Silverio_GeneralizedTaskParameterizedSkillLearning_2018,
	Torras_IncrementalLearningInTPGMM_2016
}
or
\emph{Imitation Learning}~\cite{
	Calinon_ApproachForImitationLearningOnRiemannianManifolds_2017,
	Steil_LearningRobotMotionsWithStableDynamicalSystemsUnderDiffeomorphicTransformations_2015,
	Schaal_MovementImitationWithNonlinearDynamicalSystemsInHumanoidRobots_2002
}.

At its core, the technology requires demonstration data,
which is provided through collaborative hand-guiding~\cite{Calinon_OnLearningRepresentingAndGeneralizingTaskInHumanoidRobot_2007}, teleoperation~\cite{Calinon_ImitationLearningOfPositionalAndForceSkillsDemonstratedViaKinestethicTeachingAndHapticInput_2011},
a handheld input device~\cite{Book_TimeReductionInOnlineProgramming_WandelbotExample_2022}, 
or camera-based techniques~\cite{Vogt_SystemForLearningContinuousHumanRobotInteractionsFromHumanHumanDemonstrations_2017}.
The objective is to find suitable model parameters, 
that reproduce the demonstrations accurately.
On top of single trajectory approaches,
a \emph{Motion Skill Memory}~\cite{Steil_EfficientPolicySearchWithParameterizedSkillMemory_2014}
or \emph{Memory of Motion}~\cite{Calinon_MemoryOfMotionForWarmStartingTrajectoryOptimization_2019}
is trained for further generalization and skill interpolation.
Through the resulting model,
the robot is then controlled autonomously.
Modern techniques furthermore allow the robot behavior to adapt to  variations of the task,
often referred to as \emph{Task-Parameterization}~\cite{Calinon_TutorialTPGMM_2016}.

Many applications require additional motion properties to be guaranteed.
Such verifiable safety certificates play a fundamental role in industrial production settings.
Therefore, a wide range of research on learning state-dependent dynamical systems (DS) $\dot{\DataOut} = f(\DataOut)$
focuses on enforcing Lyapunov stability at the desired attractor~\cite{Khansari_SEDS_2011, Steil_LearningRobotMotionsWithStableDynamicalSystemsUnderDiffeomorphicTransformations_2015, Saveriano_LearningStableRoboticSkillsOnRiemannianManifolds_2022}.
Note, however, that learning a DS may suffer from the high-dimensional input dimension and the restrictive stability constraints.
For example, a globally stable first-order DS can not generate a trajectory that crosses its own path.
On the contrary, time-dependent systems $\DataOut = f(t)$
provide more flexibility in representing arbitrarily shaped trajectories~\cite{Calinon_OnLearningRepresentingAndGeneralizingTaskInHumanoidRobot_2007, Calinon_LearningBimanualEndEffectorPosesFromDemonstrationsUsingTaskParameterizedDynamicalSystems_2015, Calinon_ApproachForImitationLearningOnRiemannianManifolds_2017}.

\begin{figure} 
	\includegraphics[width=0.49\linewidth]{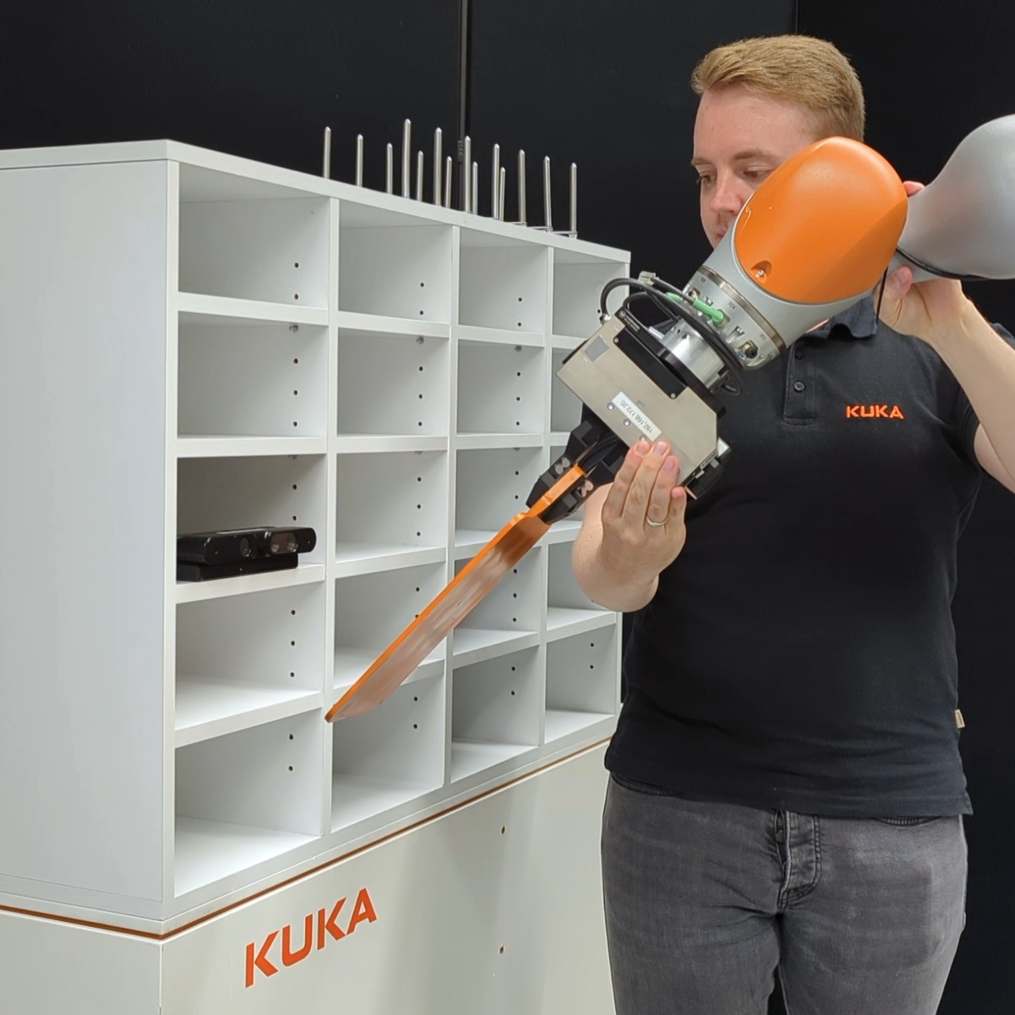}
	\hfill
	\includegraphics[width=0.49\linewidth]{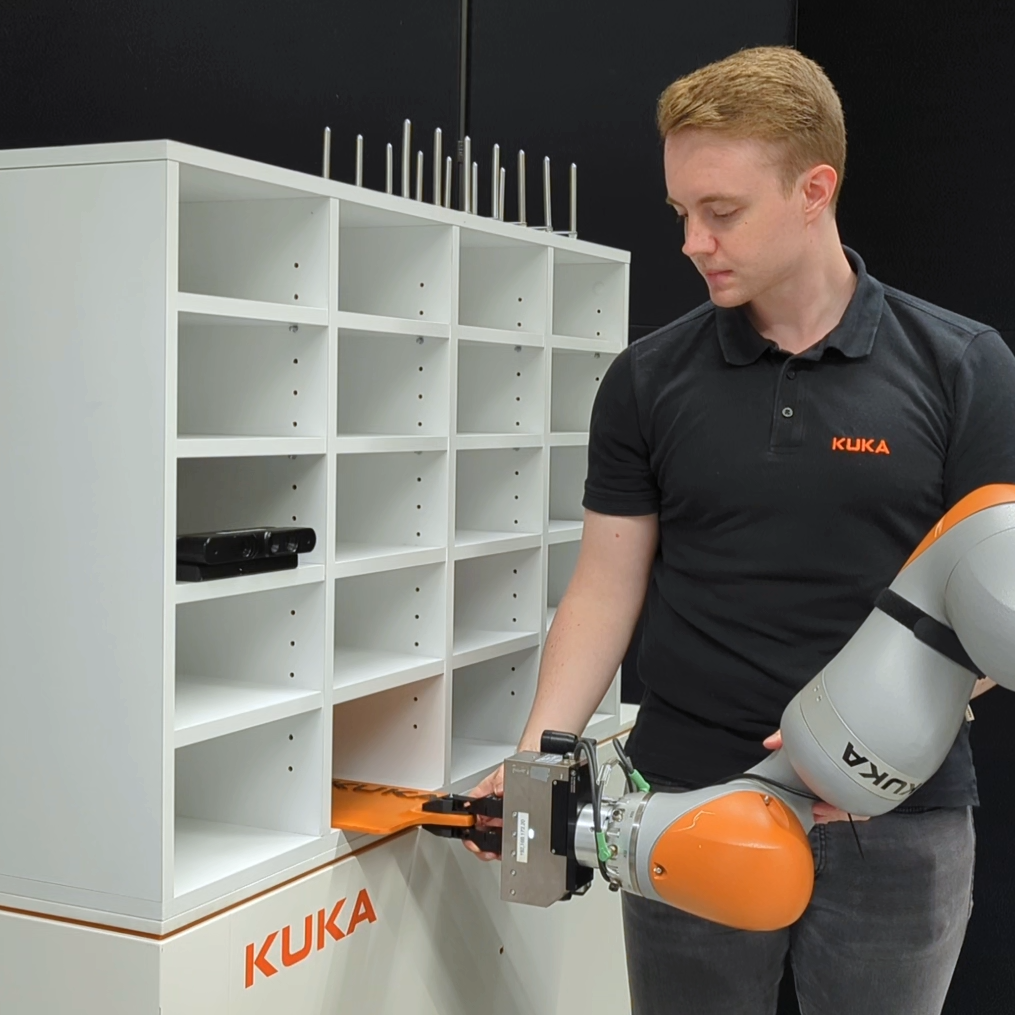}
	\caption{
		Experimental setup for a scenario with task variance.
		Demonstrations are provided through hand-guiding.
		The variable motion skill includes picking up a board out of the non-stationary rack,
		holding it in front of a camera (left)
		and inserting it into the shelf in different bins (right).
		A task-parameterized motion skill is learned without need for writing program code.
	}
	\label{fig:InspectionTaskSetup}
\end{figure}

Probabilistic methods are often used in the context of LfD~\cite{Billard_RecentAdvancesInRobotLearningFromDemonstration_2020}.
However,
state-of-the-art Gaussian Mixture Regression~(GMR) based on Expectation Maximization~(EM)~\cite{Ghahramani_SupervisedLearningFromIncompleteDataViaAnEmApproach_1993} fails 
when dealing with applications that require precise task execution at specific parts of the process. 
The main contribution of this paper is the generic formulation of a \emph{Constrained Expectation Maximization} (CEM) algorithm on Riemannian Manifolds
that enforces precision where required through prior task knowledge.
In the case of probabilistic imitation learning for time-dependent systems, 
the novel CEM enforces \emph{time-sensitive constraints} (TSC).
The approach also applies to the more generic Task-Parameterized Gaussian Mixture Model (TP-GMM)~\cite{Calinon_TutorialTPGMM_2016}.
In this letter, we demonstrate probabilistic learning of precise pick \& place motions in end-effector space, 
guaranteeing the exact reaching of desired positions and orientations.
Our approach is validated with the
KUKA LBR iiwa robot, see Fig.~\ref{fig:InspectionTaskSetup}.
Comparisons with state-of-the-art methods
reveal the advantages of our contribution.

The remainder of this work is structured as follows.
We review state-of-the-art LfD techniques in Section~\ref{sec:Background},
and reiterate GMM-based motion learning in Section~\ref{sec:Fundamentals}.
Section~\ref{sec:EnforcingTimeSensitiveConstraints} describes our main contribution on how to enforce TSC,
experimentally evaluated in Section~\ref{sec:Experiments}.
Section \ref{sec:Concolusion} concludes.

We denote scalar values with lowercase letters, e.g.~$q$,
vectors by bold lowercase letters, e.g.~$\bm{q}$ and
matrices by bold uppercase letters, e.g.~$\bm{Q}$.
Uppercase letters~$K$, $F$ and $T$ describe the number of Gaussians in a GMM, frames and inputs, respectively.
Quantities given wrt. the $f$-th local reference frame are denoted by an upper-left index, e.g.~$\Mean{}{}{f}$.
The global reference frame can be assumed if no other reference frame is indicated.
Model parameters belonging to the $k$-th component of a GMM are represented by an lower-right index, e.g.~$\Mean{k}{}{}$.

%% file: background.tex
\section{Background}
\label{sec:Background}

Over the last decades, various techniques have been developed for learning robot motions from demonstrations.
In the Dynamic Movement Primitive~(DMP) framework~\cite{
	Schaal_DynamicMovementPrimitivesFrameworkForMotorControlInHumansAndHumanoidRobots_2003,
	Billard_RecentAdvancesInRobotLearningFromDemonstration_2020},
a stable dynamical system is coupled with a learned forcing term
that encodes the desired behavior to generate point-to-point motions,
while global stability at the attractor point is guaranteed~\cite{Khansari_SEDS_2011}. Probabilistic Movement Primitives~(ProMP)~\cite{Paraschos_ProbabilisticMovementPrimitives_2013}
extend the DMP-concept and additionally capture the variation in the demonstrations.

Alternatively, the encoding of demonstration data using a GMM has become popular in recent research.
Utilizing the learned joint probability model,
GMR provides a fast and efficient way to compute a multi-variate output distribution for a given input.
The time needed for regression is independent on the number demonstrations, and hence,
robot commands are generated in an online manner~\cite{Calinon_GaussiansOnManifoldsForRobotLearning_2019}.
The GMM representation has been utilized in developing the Stable Estimator of Dynamical Systems~(SEDS)~\cite{Khansari_SEDS_2011}.
The approach introduces constraints from Lyapunov theory into the learning procedure,
thereby enforcing global stability at the desired attractor point.
Addressing the SEDS limitations resulting from the quadratic Lyapunov function,
the extended $\tau$-SEDS~\cite{Steil_LearningRobotMotionsWithStableDynamicalSystemsUnderDiffeomorphicTransformations_2015} 
applies a diffeomorphic transformation $\tau$ to the demonstration data.
An alternative diffeomorphism achieves better results~\cite{Perrin_FastDiffeomorphicMatchingToLearnGloballyAsymptoticallyStableNonlinearDynamicalSystems_2016}.
In~\cite{Nadia_PhysicallyConsistentBayesingNonParametericMixtureModelForDynamicalSystemLearning_2018} a Linear Parameter Varying Dynamical System (LPV-DS) was introduced that achieved better reproduction quality compared to SEDS.

Despite recent research advancing in learning globally stable dynamical systems (DS) from demonstrations,
the range of learnable robot motions remains limited for first-order DS models.
On the other hand,
time-dependent systems encode trajectories with arbitrary complexity.
A GMM representation for time-dependent systems was introduced in~\cite{Calinon_OnLearningRepresentingAndGeneralizingTaskInHumanoidRobot_2007}.
The Task-Parameterized Gaussian Mixture Model~(TP-GMM)~\cite{Calinon_OriginalTPGMM_2014} framework
is a direct extension of this approach for scenarios with task variance,
creating interpolation and extrapolation capabilities
by modeling the demonstration data from multiple local frame perspectives
that are adjusted to the current environment and fused at run-time~\cite{Calinon_ApproachForImitationLearningOnRiemannianManifolds_2017}.

Riemannian manifolds provide a convenient way to extend methods originally developed for Euclidean space~\cite{Calinon_GaussiansOnManifoldsForRobotLearning_2019}.
By considering unit quaternions as elements on the $S^{3}$-Sphere manifold,
the original GMM and TP-GMM framework have been extended for orientation learning in~\cite{Saveriano_LearningStableRoboticSkillsOnRiemannianManifolds_2022} and~\cite{Calinon_LearningBimanualEndEffectorPosesFromDemonstrationsUsingTaskParameterizedDynamicalSystems_2015,Calinon_ApproachForImitationLearningOnRiemannianManifolds_2017}.
In addition, the TP-GMM approach has been further improved by autonomous tuning of task-parameters 
and manual adjustment of their contribution during motion generation~\cite{Silverio_GeneralizedTaskParameterizedSkillLearning_2018}.
Identifying local frames that are irrelevant to the overall task was proposed in~\cite{Alizadeh_IdentifyingRelevantFramesInTPGMM_2016}.
Adding a weighting term during fusion to regard the importance of each perspective differently
based on their distance to the input
was investigated in~\cite{Zhu_LearningFromFewDemonstrationsWithFrameWeightedMotionGeneration_2023}.
However,
this method is limited when start and end-pose with their associated task-parameters are close to each other.

Enforcing desired properties such as specific start, target, and/or via-poses within motion generation utilizing GMR is not straightforward,
as observed in~\cite{Silverio_KMP_2019}.
To address this issue, a Bagging-GMM approach was introduced in~\cite{Ye_BaggingForGaussianMixtureRegressionInRobotLearningFromDemonstration},
where multiple base learners are trained on different subsets of the data.
Combining their predictions through suitable weights, desired time-sensitive constraints (TSC) are enforced on the generated motion.
Kernelized Movement Primitives (KMP)~\cite{Silverio_KMP_2019} on the other hand,
build upon a reference trajectory obtained from GMR
and allows to enforce desired target- or via-poses through trajectory modulation~\cite{Silverio_KMP_2019},
i.e., modification of the conditional distribution obtained from GMR by manually defining small covariances for these poses.
The approach was extended in~\cite{Silverio_ProbabilisticFrameworkForLearningGeometryBasedRobotManipulationSkills_2021}
to learn robot motions encapsulated in SPD matrices.
Alternatively, taking mathematical constraints into account during data-driven optimization
facilitates the encoding of prior knowledge into the learning process~\cite{Steil_ReliableIntegrationOfContinuousConstraintsIntoExtremeLearningMachines_2013}.

In this letter, we extend the EM algorithm to account for constraints.
The novel formulation enforces time-sensitive constraints during learning of the GMM.
Note that an adjustment of the E-step was investigated in~\cite{Graca_ExpectationMaximizationAndPosteriorConstraints_2007},
resulting in more precise control over data assignment to specific Gaussian components,
by enforcing prior knowledge on the posterior probability distribution.
For learning robot motion from demonstrations, however,
we propose a modification of the M-step instead.

%% file: fundamentals.tex
\section{Probabilistic Learning of Motion Skills}
\label{sec:Fundamentals}
Let us reiterate how to represent a motion skill as regression problem
based on demonstration data encoded in a GMM.


\subsection{Gaussian Mixture Model}
\label{sec:GMM}

Learning a time-dependent system requires a set of demonstrations~$\DataSet{} = \{ \DataAll{n}{} \}_{n=1}^{N} \inR^{D \times N}$ with $N$ samples.
Each sample~$\DataAll{n}{} \inR^{D}$
consists of a time input~$\DataIn \inR$
and a state output~$\DataOut \inR^{D\text{-}1}$, 
for which we will consider Cartesian positions, orientations, or full poses eventually.
The GMM consists of $K$~Gaussian components 
which are tuned to approximate the data's joint probability~\cite{McFarlane_GMRforLowCostMonitoringData}
\begin{equation}
	\label{eq:GMM}
	\prob(\DataAll{}{})
	=
	\prob(\DataIn, \DataOut)
	=
	\sum_{k=1}^{K} \Prior{k} \gauss(\DataAll{}{} |\Mean{k}{}{}, \Covariance{k}{}{})
\end{equation}
Each $k$-th Gaussian component is modeled by a mean vector $\Mean{k}{}{} \inR^{D}$, 
a symmetric positive-definite (SPD) covariance matrix $\Covariance{k}{}{} \inR^{D \times D}$,
and is weighted by a prior probability~$\Prior{k} \inR$.
All prior probabilities must sum up to one, i.e., $\sum_{k=1}^{K} \Prior{k} = 1$.
The $k$-th Gaussian probability density function (PDF) writes with the determinant $det(\Covariance{k}{}{})>0$
\begin{equation}
	\label{eq:GMM_PDF}
	\gauss(\DataAll{}{} | \Mean{k}{}{}, \Covariance{k}{}{})
	= 
	\frac{
		\exp\left(
			-\frac{1}{2}
			( \DataAll{}{} - \Mean{k}{}{} )^{T}
			\Covariance{k}{}{}^{-1}
			( \DataAll{}{} - \Mean{k}{}{} )
		\right)
	}{
		\sqrt{(2\pi)^{D} det(\Covariance{k}{}{})}
	}
	\,\,
\end{equation}


\subsection{Expectation Maximization}
\label{sec:GMM_EM}

The set of model parameters~$ \ModelParam{}= \left\{ \Prior{k}, \Mean{k}{}{}, \Covariance{k}{}{} \right\}_{k=1}^{K} $ 
for a GMM are commonly estimated by utilizing the EM algorithm,
thereby maximizing the log-likelihood for the given data~$\DataSet{}$
\begin{equation}
	\label{eq:LikelihoodFunction}
	\log \likelihood(\DataSet{} | \ModelParam{})
	=
	\log \prod_{n=1}^{N} \prob(\DataAll{n}{} )
	=
	\sum_{n=1}^{N} \log \prob(\DataAll{n}{} )
\end{equation}

\begin{subequations}
	The EM is an iterative procedure consisting of two steps.
	Starting from an initial set of parameters\footnote{
		For example, obtained from random sampling, k-means, or k-bins.
	},
	an \emph{expectation} step is applied
	to evaluate the responsibilities~$\Responsibility{k}{n} \inR$
	of each component generating the data
	\begin{gather}
		\label{eq:eq:GMM_eStep}
		\Responsibility{k}{n}
		=
		\frac{
			\Prior{k} \gauss( \DataAll{n}{} | \Mean{k}{}{}, \Covariance{k}{}{} ) 
		}{
			\sum_{i=1}^{K} \Prior{i} \gauss(  \DataAll{n}{} | \Mean{i}{}{}, \Covariance{i}{}{} ) 
		}
	\end{gather}
	These responsibilities are then used during a \emph{maximization} step to update the model parameters as follows
	\begin{gather}
		\label{eq:GMM_mStepPrior}
		\Prior{k}
		=
		\frac{
			\sum_{n=1}^{N} \Responsibility{k}{n}
		}{
			N 
		}
		\\
		\label{eq:GMM_mStepMu}
		\Mean{k}{}{}
		=
		\frac{
			\sum_{n=1}^{N} \Responsibility{k}{n} \DataAll{n}{}
		}{
			\sum_{n=1}^{N} \Responsibility{k}{n} 
		}
		\\
		\label{eq:GMM_mStepSigma}
		\Covariance{k}{}{}
		=
		\frac{
			\sum_{n=1}^{N} \Responsibility{k}{n} (\DataAll{n}{} - \Mean{k}{}{}) (\DataAll{n}{} - \Mean{k}{}{})^{T}
		}{
			\sum_{n=1}^{N} \Responsibility{k}{n}
		}
	\end{gather}
\end{subequations}
The EM procedure is repeated until convergence, with each iteration guaranteeing that the likelihood~\eqref{eq:LikelihoodFunction} is non-decreasing.


\subsection{Gaussian Mixture Regression}
\label{sec:GMM_GMR}

Gaussian Mixture Regression is utilized
to compute the most probable output~$\DataOut$ for a given input~$\DataIn$,
i.e., estimating the conditional probability distribution~\cite{McFarlane_GMRforLowCostMonitoringData}
\begin{equation}
	\label{eq:GMM_conditionalProbability}
	\prob(\DataOut | \DataIn, \Mean{}{}{}, \Covariance{}{}{})
	=
	\frac{\prob(\DataIn, \DataOut | \Mean{}{}{}, \Covariance{}{}{})}{\prob(\DataIn | \Mean{}{}{}, \Covariance{}{}{})}
	=
	\sum_{k=1}^{K} \ExpWeight{k}{}{} \gauss(\ExpMean{k}{}{}, \ExpCovariance{k}{}{})
\end{equation}
The resulting distribution is described by $K$ Gaussian components with
mean vectors~$\ExpMean{k}{}{} \inR^{D\text{-}1}$ and covariance matrices~$\ExpCovariance{k}{}{} \inR^{D\text{-}1 \times D\text{-}1}$,
where each $k$-th component's contribution is activated by~$0 \leq \ExpWeight{k}{}{} \leq 1$.
By decomposing means and covariances into input and output dimensions
\begin{equation}
	\label{eq:GMM_inputOutputDimensions}
	\begin{aligned}
		\Mean{k}{}{}
		=
		\begin{bmatrix} 
			\Mean{k}{\DataIn}{} \\
			\Mean{k}{\DataOut}{}
		\end{bmatrix}
		&& , &&
		\Covariance{k}{}{}
		=
		\begin{bmatrix} 
			\Covariance{k}{\DataIn\DataIn}{}  & \Covariance{k}{\DataIn \DataOut}{} \\
			\Covariance{k}{\DataIn\DataOut}{} & \Covariance{k}{\DataOut\DataOut}{}
		\end{bmatrix}
	\end{aligned} 
\end{equation}
the parameters $\ExpWeight{k}{}{}$, $\ExpMean{k}{}{}$ and $\ExpCovariance{k}{}{}$
are directly related to the parameters of the GMM as follows
\begin{subequations}
	\begin{alignat}{2}
		\label{eq:GMR_multimodalWeight}
		\ExpWeight{k}{}{}
		&=
		\frac{
			\Prior{k} \gauss(\DataIn | \Mean{k}{t}{}, \Covariance{k}{tt}{})
		}{
			\sum_{i=1}^{K} \Prior{i} \gauss(\DataIn | \Mean{i}{t}{}, \Covariance{i}{tt}{})
		}
		\\
		\label{eq:GMR_multimodalMean}
		\ExpMean{k}{}{}
		&=
		\Mean{k}{\DataOut}{} +
		\Covariance{k}{\DataIn\DataOut}{} \Covariance{k}{\DataIn\DataIn}{}^{-1}
		(\DataIn - \Mean{k}{\DataIn}{})
		\\
		\label{eq:GMR_multimodalSigma}
		\ExpCovariance{k}{}{}
		&=
		\Covariance{k}{\DataOut\DataOut}{} - 
		\Covariance{k}{\DataIn\DataOut}{} \Covariance{k}{\DataIn\DataIn}{}^{-1}
		\Covariance{k}{\DataIn\DataOut}{}
	\end{alignat}
\end{subequations}
Recall that $\Covariance{k}{\DataIn\DataIn}{}>0$ is a positive scalar, and hence $\Covariance{k}{\DataIn\DataIn}{}^{-1}>0$.

While~\eqref{eq:GMM_conditionalProbability} describes a multi-modal distribution,
a single Gaussian~$\gauss(\ExpMean{}{}{}, \ExpCovariance{}{}{}) \approx \prob(\DataOut | \DataIn, \Mean{}{}{}, \Covariance{}{}{})$
is often used to approximate the conditional probability distribution~\cite{Calinon_TutorialTPGMM_2016}.
Its mean vector $\ExpMean{}{}{}$ and covariance matrix $\ExpCovariance{}{}{}$ are given by
\begin{IEEEeqnarray}{r}
		\label{eq:GMR_Mean_Sigma}
		\ExpMean{}{}{}
		=
		\sum_{k=1}^{K} \ExpWeight{k}{}{} \ExpMean{k}{}{}
		\, \text{ and } \,
		\ExpCovariance{}{}{}
		=
		\sum_{k=1}^{K} \ExpWeight{k}{}{}
		(\ExpCovariance{k}{}{} + \ExpMean{k}{}{} \ExpMean{k}{}{}^{T})
		- \ExpMean{}{}{} \ExpMean{}{}{}^{T}
		\hspace{0.5cm}
\end{IEEEeqnarray}


\subsection{Generalization to Riemannian Manifold}
\label{sec:ExtensionToUnitQuaternions}

GMM and GMR are originally defined for vectors in Euclidean space,
whereas end-effector orientations are often described using unit quaternions,
which represent elements on the $S^{3}$-Sphere\footnote{
	Note that $-q$ and $q$ describe the same orientation,
	because the $S^{3}$-Sphere covers the range of all unit quaternions twice.
},
a 4-dimensional Riemannian manifold~\cite{Calinon_ApproachForImitationLearningOnRiemannianManifolds_2017}.
Therefore, the approach has been extended for the $S^{d}$-Sphere,
by considering Gaussians with mean vectors $\Mean{}{}{} \in \Manifold$ on the manifold
and covariance matrices $\Covariance{}{}{} \in \Tangent{\Mean{}{}{}}$ defined in the tangent space associated with the mean vector.
The latter describes the dispersion of the original data projected onto the tangent space \cite{Torras_3dHumanPoseTrackingPriorsUsingGeodesicMixtureModels_2017}.
Equation~\eqref{eq:GMM_PDF} is generalized as
\begin{IEEEeqnarray}{r}
		\label{eq:ManifoldGaussian}
		\gauss(\DataManifold{}{} | \Mean{k}{}{}, \Covariance{k}{}{})
		= 
		\frac{
			exp\left( 
			-\frac{1}{2}
			\logmap{\DataManifold{}{}}{\Mean{k}{}{}}^{T}
			\Covariance{k}{}{}^{-1}
			\logmap{\DataManifold{}{}}{\Mean{k}{}{}}
			\right) 
		}{
			\sqrt{(2\pi)^{D} det(\Covariance{k}{}{})}
		}
	\hspace{0.5cm}
\end{IEEEeqnarray}
where $\logmap{\DataManifold{}{}}{\Mean{}{}{}{}} \!:\! \Manifold \! \mapsto \! \Tangent{\Mean{}{}{}}$
describes the Logarithmic Map that projects an element~$\DataManifold{}{} \in \Manifold$ from the $S^{d}$-Sphere
onto the tangent space~$\Tangent{\Mean{}{}{}}$
associated with point~$\Mean{}{}{} \in \Manifold$.
The inverse operation is the Exponential Map~$\expmap{\DataTangent{}{}{}}{\Mean{}{}{}} : \Tangent{\Mean{}{}{}} \mapsto \Manifold$
and projects an element~$\DataTangent{}{}{} \in \Tangent{\Mean{}{}{}}$
from the tangent space~$\Tangent{\Mean{}{}{}}$
back onto the $S^{d}$-Sphere.
Additionally,
the parallel transport operation $\parallelTransport{\DataManifold{}{}}{\DataManifoldTmp{}{}}(\DataTangent{}{}{}) : \Tangent{\DataManifold{}{}} \mapsto \Tangent{\DataManifoldTmp{}{}}$
has to be applied,
when combining geometric information from different tangent spaces.
For more details, we refer the interested reader to~\cite{Calinon_ApproachForImitationLearningOnRiemannianManifolds_2017} and~\cite{Calinon_GaussiansOnManifoldsForRobotLearning_2019}.

\begin{subequations}

The generalized Expectation Maximization algorithm computes responsibilities $\Responsibility{k}{n}$ within the \emph{expectation} step as
\begin{equation}
	\label{eq:GMM_eStepManifold}
	\Responsibility{k}{n}
	=
	\frac{
		\Prior{k} \gauss( \logmap{\DataManifold{n}{}}{\Mean{k}{}{}} | \Mean{k}{}{}, \Covariance{k}{}{} ) 
	}{
		\sum_{i=1}^{K} \Prior{i} \gauss( \logmap{\DataManifold{n}{}}{\Mean{i}{}{}} | \Mean{i}{}{}, \Covariance{i}{}{} ) 
	}
\end{equation}
and updates means~$\Mean{k}{}{}$ through few Gauss-Newton iterations
\begin{equation}
	\label{eq:GMM_mStepManifoldMean}
	\DataTangent{k}{}{}
	=
	\frac{
		\sum_{n=1}^{N} \Responsibility{k}{n} \logmap{\DataManifold{n}{}}{\Mean{k}{}{}}
	}{
		\sum_{n=1}^{N} \Responsibility{k}{n} 
	}
	\,\,\,\,  ,  \,\,\,\,
	\Mean{k}{}{} \leftarrow \expmap{\DataTangent{k}{}{}}{\Mean{k}{}{}}
\end{equation}
After convergence,
the new covariance matrix is obtained by projecting all samples~$\DataManifold{}{}$ from $\DataSet{}$ onto the tangent space~$\Tangent{\Mean{k}{}{}}$
\begin{equation}
	\label{eq:GMM_mStepManifoldCovariance}
	\Covariance{k}{}{}
	=
	\frac{
		\sum_{n=1}^{N} \Responsibility{k}{n}
		\logmap{\DataManifold{n}{}}{\Mean{k}{}{}}
		\logmap{\DataManifold{n}{}}{\Mean{k}{}{}}^{T}
	}{
		\sum_{n=1}^{N} \Responsibility{k}{n}
	}
\end{equation}
Prior probabilities~$\Prior{k}$ are computed according to \eqref{eq:GMM_mStepPrior}.
\end{subequations}

\begin{subequations}
The generalized GMR utilizes
$\CovariancePT{k}{}{} = \parallelTransport{\Mean{k}{}{}}{\ExpMean{}{}{}}(\Covariance{k}{}{})$
which describes the covariance~$\Covariance{k}{}{}$
that has been parallel transported from $\Tangent{\Mean{k}{}{}}$ to $\Tangent{\ExpMean{}{}{}}$,
in order to compensate for the relative rotation between tangent spaces.
Activation weights~$\ExpWeight{k}{}{}$ are given by~\eqref{eq:GMR_multimodalWeight}.
To compute the most probable output~$\ExpMean{}{}{}$ for a given input,
each Gaussian's prediction $\DataTangent{k}{}{}$ is locally evaluated in the tangent space of~$\ExpMean{}{}{}$
\begin{equation}
	\label{eq:GmrMeanManifoldMultimodal}
	\DataTangent{k}{}{}
	=
	\logmap{\Mean{k}{\DataOut}{}}{\ExpMean{}{}{}}
	+
	\CovariancePT{k}{\DataOut \DataIn}{}
	\CovariancePT{k}{\DataIn \DataIn}{}^{-1}
	\left( \DataIn - \Mean{k}{\DataIn}{} \right)
\end{equation}
which is updated through few Gauss-Newton iterations
\begin{equation}
	\label{eq:GmrMeanManifoldLocal}
	\DataTangent{}{}{}
	=
	\sum_{k=1}^{K} \ExpWeight{k}{}{} \DataTangent{k}{}{}
	\quad , \quad
	\ExpMean{}{}{} \leftarrow \expmap{\DataTangent{}{}{}}{\ExpMean{}{}{}}
\end{equation}
After convergence, the $k$-th covariance matrix~$\ExpCovariance{k}{}{}$ is computed locally in the tangent space of~$\ExpMean{}{}{}$
\begin{equation}
	\label{eq:GmrCovarianceManifoldMultimodal}
	\ExpCovariance{k}{}{}
	=
	\CovariancePT{k}{\DataOut \DataOut}{}
	-
	\CovariancePT{k}{\DataOut \DataIn}{}
	\CovariancePT{k}{\DataIn \DataIn}{}^{-1}
	\CovariancePT{k}{\DataIn \DataOut}{}
	\\
\end{equation}
enabling us to approximate~$\ExpCovariance{}{}{}$
\begin{equation}
	\label{eq:GmrCovarianceManifoldLocal}
	\ExpCovariance{}{}{}
	=
	\sum_{k=1}^{K}
	\ExpWeight{k}{}{}
	\left( \ExpCovariance{k}{}{} + \DataTangent{k}{}{} \DataTangent{k}{}{}^{T} \right)
	-
	\DataTangent{}{}{}
	\DataTangent{}{}{}^{T}
\end{equation}
\end{subequations}

This generalized approach does not only apply to unit quaternions on the $S^{3}$-Sphere, 
but also to other Riemannian manifolds~\cite{Jaquier_GaussianMixtureRegressionOnSPDMatrices_2017}.
We obtain the Euclidean case when choosing an identity mapping as logarithmic and exponential map.
Then, the Gauss-Newton iterations converge instantly and the parallel transport operation is not necessary.


\subsection{Task-Parameterization}
\label{sec:TPGMM}

When learning robot motions from demonstrations,
projection of data onto a suitable reference frame can significantly improve the learning result.
The so-called Task-Parameterized Gaussian Mixture Model (TP-GMM) framework developed by Calinon~\cite{Calinon_TutorialTPGMM_2016} 
exploits $F>1$ reference frames to encode different parts of a variable skill.
These frames are referred to as task-parameters $\{ \TaskA{f}, \TaskB{f} \}_{f=1}^{F}$,
specified by a position~$\TaskB{f}$ with an orientation~$\TaskA{f}$,
and describe a particular task scenario.
They allow projection of the demonstration data~$\DataSet{}$ from the global frame onto different local viewpoints~$\DataSet{f}$.
Local GMMs are trained, and model parameters become $\{ \Prior{k}, \{ \Mean{k}{}{f} , \Covariance{k}{}{f} \}_{f=1}^{F} \}_{k=1}^{K}$,
under the assumption that all data samples were generated from the same source~\cite{Calinon_TutorialTPGMM_2016}.

During run-time, 
new task-parameters that describe the current scenario are supplied.
Then, for each local perspective $f \in \{1,\cdots,F\}$,
predictions $\gauss(\DataIn_{n} | \ExpMeanLocal{n}{}{f} , \ExpCovarianceLocal{n}{}{f})$
are computed
according to~\eqref{eq:GmrMeanManifoldMultimodal}-\eqref{eq:GmrCovarianceManifoldLocal}
for all time inputs $\DataIn_{n} \in \{ \DataIn_{1}, \cdots, \DataIn_{T} \}$.
Utilizing the task-parameters,
these predictions are projected into the
global perspective~$\gauss(\DataIn_{n} | \ExpMean{n}{}{f} , \ExpCovariance{n}{}{f})$ 
according to
\begin{subequations}
	\label{eq:TPGMM_ManifoldLocalToGlobal}
	\begin{alignat}{2}
		\ExpMean{n}{}{f}
		&=
		\expmap{ \logmap{\ExpMeanLocal{n}{}{f}}{\origin{}} }{\TaskA{f}}
		\\
		\ExpCovariance{n}{}{f}
		&=
		\parallelTransport{\ExpMeanLocal{n}{}{f}}{\ExpMean{n}{}{f}}(\ExpCovarianceLocal{n}{}{f})
	\end{alignat}
\end{subequations}
where $\origin{}$ describes the identity orientation.

\begin{subequations}
With all predictions being expressed in the same global reference frame,
the information from each perspective is combined using the Product of Gaussians.
Means~$\ExpMean{n}{}{}$ are computed using Gauss-Newton iterations according to
\begin{gather}
	\label{eq:TPGMM_ManifoldProductOfGaussiansSumCovariance}
	\CovariancePT{n}{\Sigma}{}
	= 
	\sum_{f=1}^{F} \left( \CovariancePT{n}{}{f}^{-1} \right)
	\\
	\label{eq:TPGMM_ManifoldProductOfGaussiansSumMean}
	\DataTangent{n}{\Sigma}{}
	=
	\sum_{f=1}^{F} \left( \CovariancePT{n}{}{f}^{-1} \logmap{\ExpMean{n}{}{f}}{\ExpMean{n}{}{}} \right)
	\\
	\label{eq:TPGMM_ManifoldProductOfGaussiansMean}
	\DataTangent{n}{}{}
	=
	\left( \CovariancePT{n}{\Sigma}{} \right)^{-1}
	\DataTangent{n}{\Sigma}{}
	\quad , \quad
	\ExpMean{n}{}{} \leftarrow \expmap{\DataTangent{n}{}{}}{\ExpMean{n}{}{}}
\end{gather}
where $\CovariancePT{n}{}{f} = \parallelTransport{\ExpMean{n}{}{f}}{\ExpMean{n}{}{}}(\ExpCovariance{n}{}{f})$
describes covariance~$\ExpCovariance{n}{}{f}$,
which has been parallel transported from $\Tangent{\ExpMean{n}{}{f}}$ to $\Tangent{\ExpMean{n}{}{}}$.
After convergence, covariances~$\ExpCovariance{n}{}{}$ are computed as 
$\ExpCovariance{n}{}{} = ( \CovariancePT{n}{\Sigma}{} )^{-1}$.
\end{subequations}
The result is a single prediction~$\gauss(\DataIn_{n} | \ExpMean{n}{}{}, \ExpCovariance{n}{}{})$
in the global perspective
for each input $\DataIn_{n} \in \{ \DataIn_{1}, \cdots, \DataIn_{T} \}$

%% file: enforcing_constraints.tex
\section{Enforcing Time-Sensitive Constraints}
\label{sec:EnforcingTimeSensitiveConstraints}

In many industrial applications,
an end-effector trajectory should start and stop exactly at predefined times at the desired start and target pose.
However, this can not be guaranteed when learning a time-dependent system utilizing standard GMM and GMR.
Therefore, we propose to include task knowledge into the learning process
by enforcing time-sensitive constraints (TSC).

\subsection{Learning a GMM subject to Constraints}
\label{sec:TSCinGMM}

To enforce a desired pose~$\DataOut_{\text{des}}$ at time~$\DataIn_{\text{des}}$,
we investigate the single-Gaussian approximation of the conditional probability distribution.
According to~\eqref{eq:GMR_Mean_Sigma}, 
the approximated mean $\ExpMean{}{}{}$ results from the sum
of each Gaussian's prediction~$\ExpMean{k}{}{}$
multiplied by an activation weight~$\ExpWeight{k}{}{}$.
Therefore,
we aim for a model with conditions
\begin{subequations}
	\label{eq:desired_properties}
	\begin{equation}
		\label{eq:desired_properties_mean}
		\ExpMean{\IndexGaussian{}}{}{} = \DataOut_{\text{des}}
	\end{equation}
	\begin{equation}
		\label{eq:desired_properties_activation}
		\ExpWeight{\IndexGaussian{}}{}{} \approx 1
		\quad\text{,}\quad
		\ExpWeight{k}{}{} \approx 0, \forall k \neq \IndexGaussian{}
	\end{equation}
\end{subequations}
where $\IndexGaussian{} \in \{1, \cdots, K\}$
describes the index of the Gaussian component that is used to enforce the desired pose.
Before describing our novel approach, we will discuss relevant lemmas.
\begin{lemma}
\label{lemma:constrained_mean}
	If $\Mean{\IndexGaussian{}}{\DataOut}{} = \DataOut_{\text{des}}$
	and $\Mean{\IndexGaussian{}}{\DataIn}{} = \DataIn_{\text{des}}$,
	then $\ExpMean{\IndexGaussian{}}{}{} = \DataOut_{\text{des}}$
	when querying the model for the input $\DataIn = \DataIn_{\text{des}}$.
\end{lemma}
\begin{proof}
	This is trivial through~\eqref{eq:GMR_multimodalMean}, i.e., for the $\IndexGaussian{}$-th Gaussian it yields
	$
	\ExpMean{\IndexGaussian{}}{}{}
			=
			\Mean{\IndexGaussian{}}{\DataOut}{} +
			\Covariance{\IndexGaussian{}}{\DataIn\DataOut}{} \Covariance{\IndexGaussian{}}{\DataIn\DataIn}{}^{-1}
			(\DataIn_{\text{des}} - \Mean{\IndexGaussian{}}{\DataIn}{})
			= \DataOut_{\text{des}}
	$.
\end{proof}

\begin{lemma}
\label{lemma:condition_activation_ratio}
	The $\IndexGaussian{}$-th component is activated with
	$\ExpWeight{\IndexGaussian{}}{}{} \approx 1$,
	$\ExpWeight{k}{}{} \approx 0, \forall k \neq \IndexGaussian{}$,
	if
	$
		\gauss(\DataIn | \Mean{\IndexGaussian{}}{\DataIn}{} , \Covariance{\IndexGaussian{}}{\DataIn\DataIn}{})
		\gg
		\gauss(\DataIn | \Mean{k}{\DataIn}{} , \Covariance{k}{\DataIn\DataIn}{}), \forall k \neq \IndexGaussian{}
	$.
\end{lemma}
\begin{proof}
	Given
	$
	\gauss(\DataIn | \Mean{\IndexGaussian{}}{\DataIn}{} , \Covariance{\IndexGaussian{}}{\DataIn\DataIn}{})
	\gg
	\gauss(\DataIn | \Mean{k}{\DataIn}{} , \Covariance{k}{\DataIn\DataIn}{}), \forall k \neq \IndexGaussian{}
	$,
	the ratio in~\eqref{eq:GMR_multimodalWeight} for determining the activation weight $\ExpWeight{\IndexGaussian{}}{}{}$
	is dominated by $\gauss(\DataIn | \Mean{\IndexGaussian{}}{\DataIn}{} , \Covariance{\IndexGaussian{}}{\DataIn\DataIn}{})$.
	For the $\IndexGaussian{}$-th component holds
	\begin{equation*}
		\ExpWeight{\IndexGaussian{}}{}{}
		=
		\frac
		{
			\Prior{\IndexGaussian{}}
			\gauss( 
			\DataIn |
			\Mean{\IndexGaussian{}}{\DataIn}{} ,
			\Covariance{\IndexGaussian{}}{\DataIn \DataIn}{}
			)
		}{
			\sum_{i=1}^{K}
			\Prior{i}
			\gauss( 
			\DataIn |
			\Mean{i}{\DataIn}{} ,
			\Covariance{i}{\DataIn \DataIn}{}
			)
		}
		\approx
		\frac
		{
			\Prior{\IndexGaussian{}}
			\gauss( 
			\DataIn |
			\Mean{\IndexGaussian{}}{\DataIn}{} ,
			\Covariance{\IndexGaussian{}}{\DataIn \DataIn}{}
			)
		}{
			\Prior{\IndexGaussian{}}
			\gauss( 
			\DataIn |
			\Mean{\IndexGaussian{}}{\DataIn}{} ,
			\Covariance{\IndexGaussian{}}{\DataIn \DataIn}{}
			)
		}
		\approx
		1
	\end{equation*}
	and for all other components $k \neq \IndexGaussian{}$ holds
	\begin{equation*}
		\ExpWeight{k}{}{}
		=
		\frac
		{
			\Prior{k}
			\gauss( 
			\DataIn |
			\Mean{k}{\DataIn}{} ,
			\Covariance{k}{\DataIn \DataIn}{}
			)
		}{
			\sum_{i=1}^{K}
			\Prior{i}
			\gauss( 
			\DataIn |
			\Mean{i}{\DataIn}{} ,
			\Covariance{i}{\DataIn \DataIn}{}
			)
		}
		\approx
		\frac
		{
			\Prior{k}
			\gauss( 
			\DataIn |
			\Mean{k}{\DataIn}{} ,
			\Covariance{k}{\DataIn \DataIn}{}
			)
		}{
			\Prior{\IndexGaussian{}}
			\gauss( 
			\DataIn |
			\Mean{\IndexGaussian{}}{\DataIn}{} ,
			\Covariance{\IndexGaussian{}}{\DataIn \DataIn}{}
			)
		}
		\approx
		0
		\qedhere
	\end{equation*}
\end{proof}
	
In the following, let us denote
\begin{equation}
\label{eq:newPDE}
	\gauss(\DataIn | \Mean{k}{\DataIn}{}, \Covariance{k}{\DataIn\DataIn}{})
	= 
	\frac{ 1 }{ \denominator_{k} } \exp\left( \expArg_{k} \right) \inR
\end{equation}
with a constant factor
	$
		\denominator_{k}
		=
		\sqrt{2 \pi \det( \Covariance{k}{\DataIn\DataIn}{} ) }
	$
and the exponential argument
	$
		\expArg_{k}
		=
		-\frac{1}{2}
		( \DataIn - \Mean{k}{\DataIn}{} )^{T}
		\Covariance{k}{\DataIn\DataIn}{}^{-1}
		( \DataIn - \Mean{k}{\DataIn}{} )
	$.
Recall that the covariance $\Covariance{k}{}{}$ is a SPD matrix,
i.e., $\det(\Covariance{k}{}{}) > 0$,
hence also $\det( \Covariance{k}{\DataIn\DataIn}{} ) > 0$ and $\denominator_{k} >0$.
The scalar $\expArg_{k}$ constitutes a negative squared distance
between input $\DataIn$ and mean $\Mean{k}{\DataIn}{}$
(scaled by $\frac{1}{2}\Covariance{k}{\DataIn\DataIn}{}^{-1} > 0$).
It holds $\expArg_{k} \leq 0$
and $0 < \exp(\expArg_{k}) \leq 1$.

\begin{lemma}
	\label{lemma:condition_PDF_maximized}
	If $\Mean{\IndexGaussian{}}{\DataOut}{} = \DataOut_{\text{des}}$
	and $\Mean{\IndexGaussian{}}{\DataIn}{} = \DataIn_{\text{des}}$
	(see Lemma~\ref{lemma:constrained_mean}),
	then $\gauss(\DataIn | \Mean{\IndexGaussian{}}{\DataIn}{} , \Covariance{\IndexGaussian{}}{\DataIn\DataIn}{})$
	reaches its maximum at $\DataIn =  \DataIn_{\text{des}}$.
\end{lemma}
\begin{proof}
	In the case $\DataIn = \DataIn_{\text{des}}$ holds $\expArg_{\IndexGaussian{}} = 0$ and $\exp(\expArg_{\IndexGaussian{}}) = 1$.
	Furthermore, $\denominator_{\IndexGaussian{}} >0$ is a constant positive factor.
	Therefore, $\gauss(\DataIn | \Mean{\IndexGaussian{}}{\DataIn}{} , \Covariance{\IndexGaussian{}}{\DataIn\DataIn}{})$
	has its maximum at $\DataIn = \DataIn_{\text{des}}$.
\end{proof}

\begin{lemma}
	\label{lemma:covariance_scaling}
	Assume $\Mean{\IndexGaussian{}}{\DataOut}{} = \DataOut_{\text{des}}$
	and $\Mean{\IndexGaussian{}}{\DataIn}{} = \DataIn_{\text{des}}$
	as discussed in Lemma \ref{lemma:constrained_mean}.
	Multiplying~$\Covariance{k}{\DataIn\DataIn}{}$
	in the covariance from the $k$-th Gaussian component 
	with a scaling factor $\scaling{k} \in (0,1)$
	increases the activation $\ExpWeight{\IndexGaussian{}}{}{}$ of the $\IndexGaussian{}$-th component
	at $\DataIn = \DataIn_{\text{des}}$.
\end{lemma}
\begin{proof}
	Note $(\scaling{k} \Covariance{k}{\DataIn\DataIn}{})^{-1} = \frac{1}{\scaling{k}} \Covariance{k}{\DataIn\DataIn}{}^{-1}$.
	Due to $0 \! < \! \scaling{k} \! < \! 1$ yields $\frac{1}{\scaling{k}} > 1$.
	With $\expArg_{k} \leq 0$ holds $\exp(\frac{1}{\scaling{k}} \expArg_{k}) \leq \exp(\expArg_{k})$.
	Also, we have $\det(\scaling{k} \Covariance{k}{\DataIn\DataIn}{}) = \scaling{k} \det(\Covariance{k}{\DataIn\DataIn}{})$.
	Hence, the constant factor
	$
		\frac{1}{\denominator_{k}}
	$
	in~\eqref{eq:newPDE} becomes 
	$
		\frac{1}{\sqrt{\scaling{k}} \denominator_{k}}
		>
		\frac{1}{\denominator_{k}}
	$.
	In the following, we need to distinguish two cases, $k = \IndexGaussian{}$ and $k \neq \IndexGaussian{}$.
	
	Consider first $k = \IndexGaussian{}$: 
	For $\DataIn = \DataIn_{\text{des}}$,
	the $\IndexGaussian{}$-th component has an exponential argument $\expArg_{\IndexGaussian{}} = 0$.
	Therefore, it holds
	$
		\gauss(\DataIn_{\text{des}} | \Mean{\IndexGaussian{}}{\DataIn}{} , \scaling{\IndexGaussian{}} \Covariance{\IndexGaussian{}}{\DataIn\DataIn}{})
		>
		\gauss(\DataIn_{\text{des}} | \Mean{\IndexGaussian{}}{\DataIn}{} , \Covariance{\IndexGaussian{}}{\DataIn\DataIn}{})
	$.
	
	Consider now $k \neq \IndexGaussian{}$ (and $\Mean{k}{\DataIn}{} \neq \DataIn_{\text{des}}$):
	For $\DataIn = \DataIn_{\text{des}}$
	the exponential argument $\expArg_{k} < 0$ is strictly negative and
	after scaling it yields $\scaling{k} \expArg_{k} < \expArg_{k} < 0$.
	Therefore, it holds
	$
		\gauss(\DataIn_{\text{des}} | \Mean{k}{\DataIn}{} , \scaling{k} \Covariance{k}{\DataIn\DataIn}{})
		<
		\gauss(\DataIn_{\text{des}} | \Mean{k}{\DataIn}{} , \Covariance{k}{\DataIn\DataIn}{})
		, \forall k \neq \IndexGaussian{}
	$.
	
	Thus,
	in both cases $k = \IndexGaussian{}$ and $k \neq \IndexGaussian{}$,
	the scaling with factor $\scaling{k} \in (0,1)$ causes the activation~$\ExpWeight{\IndexGaussian{}}{}{}$
	to increase at $\DataIn = \DataIn_{\text{des}}$.
\end{proof}

\begin{theorem}
	Given $\Mean{\IndexGaussian{}}{\DataOut}{} = \DataOut_{\text{des}}$
	and $\Mean{\IndexGaussian{}}{\DataIn}{} = \DataIn_{\text{des}}$,
	there exists a set of $K$ scaling factors $\scaling{k} \in (0,1]$,
	i.e. $\scaling{k} \Covariance{k}{\DataIn\DataIn}{}, \forall k \in K$,
	such that activations become
	$\ExpWeight{\IndexGaussian{}}{}{} \approx 1$,
	$\ExpWeight{k}{}{} \approx 0, \forall k \neq \IndexGaussian{}$
	when querying the model for the input $\DataIn = \DataIn_{\text{des}}$.
\end{theorem}
\begin{proof}
	If  
	$
		\gauss(\DataIn | \Mean{\IndexGaussian{}}{\DataIn}{} , \Covariance{\IndexGaussian{}}{\DataIn\DataIn}{})
		\gg
		\gauss(\DataIn | \Mean{k}{\DataIn}{} , \Covariance{k}{\DataIn\DataIn}{}), \forall k \neq \IndexGaussian{}
	$,
	it holds
	$\ExpWeight{\IndexGaussian{}}{}{} \approx 1$,
	$\ExpWeight{k}{}{} \approx 0, \forall k \neq \IndexGaussian{}$
	(see Lemma~\ref{lemma:condition_activation_ratio}).
	Given $\Mean{\IndexGaussian{}}{\DataOut}{} = \DataOut_{\text{des}}$
	and $\Mean{\IndexGaussian{}}{\DataIn}{} = \DataIn_{\text{des}}$,
	$\gauss(\DataIn | \Mean{\IndexGaussian{}}{\DataIn}{} , \Covariance{\IndexGaussian{}}{\DataIn\DataIn}{})$
	is maximized at $\DataIn = \DataIn_{\text{des}}$
	(see Lemma~\ref{lemma:condition_PDF_maximized}).
	Applying a scaling $\scaling{k} \in (0,1), \forall k \in K$
	to the covariance of a Gaussian
	increases $\gauss(\DataIn_\text{des} | \Mean{\IndexGaussian{}}{\DataIn}{} , \Covariance{\IndexGaussian{}}{\DataIn\DataIn}{})$
	and decreases $\gauss(\DataIn_\text{des} | \Mean{k}{\DataIn}{} , \Covariance{k}{\DataIn\DataIn}{}), \forall k \neq \IndexGaussian{}$,
	thus increases $\ExpWeight{\IndexGaussian{}}{}{}$
	(see Lemma~\ref{lemma:covariance_scaling}).
	Therefore,
	scaling factors $\scaling{k} \in (0,1], \forall k \in K$ exist,
	such that overlapping regions of influence,
	i.e., $\ExpWeight{k}{}{} \gg 0, \exists k \neq \IndexGaussian{}$,
	are removed at $\DataIn = \DataIn_{\text{des}}$
	and the $\IndexGaussian{}$-th component is activated with
	$\ExpWeight{\IndexGaussian{}}{}{} \approx 1$,
	$\ExpWeight{k}{}{} \approx 0, \forall k \neq \IndexGaussian{}$
\end{proof}

\begin{figure}[!t]
	\removelatexerror
	\begin{algorithm}[H]
		\caption{Constrained EM Algorithm~(CEM)}
		\label{alg:ConstrainedEM}
		\textbf{Input:}
		\\
		\hspace{\algorithmicindent}
		initialized model parameters
		$\{ \Prior{k} , \Mean{k}{}{} , \Covariance{k}{}{} \}_{k=1}^{K}$
		\\
		\hspace{\algorithmicindent}
		demonstration data
		$\{ \DataAll{n}{} \}_{n=1}^{N}$
		\\
		\hspace{\algorithmicindent}
		fixed Gaussian's index
		$\IndexGaussian{}$ 
		\\
		\hspace{\algorithmicindent} desired pose~$\DataOut_{\text{des}}$, time~$\DataIn_{\text{des}}$ and threshold~$\epsilon$ 
		\\
		\textbf{Output:}
		\\
		\hspace{\algorithmicindent}
		optimized model parameters
		$\{ \Prior{k} , \Mean{k}{}{} , \Covariance{k}{}{} \}_{k=1}^{K}$
		\\
		\While(){not converged}
		{
			\For(\hfill\algComment{E-step}){$n \gets 1$ to $N$}
			{        
				\For(){$k \gets 1$ to $K$}
				{
					compute responsibility $\Responsibility{k}{n}$~\eqref{eq:GMM_eStepManifold} \\
				}   
			}
			\For(\hfill\algComment{M-step}){$k \gets 1$ to $K$}
			{
				compute prior probability $\Prior{k}$~\eqref{eq:GMM_mStepPrior} \\
				\uIf(\hfill\algComment{fixed mean}){$k$ is $\IndexGaussian{}$}
				{
					set $\Mean{\IndexGaussian{}}{\DataIn}{} = \DataIn_{\text{des}}$ \\
					set $\Mean{\IndexGaussian{}}{\DataOut}{} = \DataOut_{\text{des}}$ \\
				}
				\Else
				{
					compute mean $\Mean{k}{}{}$~\eqref{eq:GMM_mStepManifoldMean} \\
				}
				compute covariance $\Covariance{k}{}{}$~\eqref{eq:GMM_mStepManifoldCovariance} \\
			} 
		}
		enforce $1 \approx \ExpWeight{\IndexGaussian{}}{}{}$~\eqref{eq:CovarianceOptimization}
		\hfill\algComment{covariance adjustment} \\
	\end{algorithm}
\end{figure}

\textbf{Proposed EM Modification:}
We suggest a Constrained~EM algorithm that includes task knowledge into the learning process,
as summarized in Algorithm~\ref{alg:ConstrainedEM}.
To find an optimal GMM with the desired properties specified in~\eqref{eq:desired_properties},
we propose to modify the \emph{maximization} step of the EM algorithm,
such that the $\IndexGaussian{}$-th mean~$\Mean{\IndexGaussian{}}{}{}$ remains fixed in each iteration
with $\Mean{\IndexGaussian{}}{\DataOut}{} = \DataOut_{\text{des}}$ and $\Mean{\IndexGaussian{}}{\DataIn}{} = \DataIn_{\text{des}}$.
Note that the property of the non-decreasing likelihood measure in each EM iteration is not affected by this modification,
see Theorem \ref{theorem:non_decreasing_fixed_mean} for details.

After the EM has converged,
we compute and apply scaling factors $\scaling{k} \in (0,1], \forall k \in K$
for each covariance matrix\footnote{
	We scale the full covariance matrix $\Covariance{k}{}{}$,
	as scaling $\Covariance{k}{\DataIn\DataIn}{}$
	would eventually violate the SPD matrix property of $\Covariance{k}{}{}$, refer to Sylvester's criterion.
	All statements above remain valid.
}.
This modifies their activations~$\ExpWeight{k}{}{}$
such that $\ExpWeight{\IndexGaussian{}}{}{} \geq 1- (K-1) \epsilon$, $\ExpWeight{k}{}{} \leq \epsilon, \forall k \neq \IndexGaussian{}$,
considering a small predefined threshold~$\epsilon$ for numerical reasons.
Appropriate scaling factors $\scaling{k}$ are obtained by solving an optimization problem as follows
\begin{equation}
	\label{eq:CovarianceOptimization}
	\begin{aligned}
		&\!\Max_{\scaling{1}, \cdots, \scaling{K}} & & \sum_{n=1}^{N} \log \sum_{k=1}^{K} \Prior{k} \gauss(\DataAll{n}{} | \Mean{k}{}{} , \scaling{k} \Covariance{k}{}{})
		\\
		&\text{s.t.} & \quad &
		\left\{ 
		\frac
		{
			\Prior{k}
			\gauss( 
			\Mean{\IndexGaussian{}}{\DataIn}{} |
			\Mean{k}{\DataIn}{} ,
			\scaling{k} \Covariance{k}{\DataIn \DataIn}{}
			)
		}{
			\sum_{i=1}^{K}
			\Prior{i}
			\gauss( 
			\Mean{\IndexGaussian{}}{\DataIn}{} |
			\Mean{i}{\DataIn}{} ,
			\scaling{i} \Covariance{i}{\DataIn \DataIn}{}
			)
		}
		\leq
		\epsilon
		\right\}_{\forall k \neq \IndexGaussian{}}^{}
		\\
		&                  & &
		0 < \scaling{k} \leq 1 \quad \forall k \in K
	\end{aligned}
\end{equation}

\textbf{Proposed Initialization:}
Before executing the CEM, we need to specify the $\IndexGaussian{}$-th component that enforces the TSC.
We sort the demonstration data based on timestamps and split the samples into $K$ equal bins.
Based on these assignments,
initial means and covariances are computed.
The most suitable $\IndexGaussian{}$-th component is then identified through 
\begin{equation}
	\label{eq:IdentifyClosestComponent}
	\IndexGaussian{}
	=
	\Argmin_{k} \left( \{ \norm{ \Mean{k}{\DataOut}{} - \DataOut_{\text{des}} } \}_{k=1}^{K} \right) 
\end{equation}

\textbf{Remark:}
It is straightforward to extend the described approach for
enforcing multiple desired poses
at multiple given times,
by constraining multiple Gaussian components.


\subsection{Learning a TP-GMM subject to Constraints}
\label{sec:TSCinTPGMM}

We here extend GMM learning with TSC to the TP-GMM framework.
Due to the task-parameterized projection of data into local perspectives~$\left\{ \DataSet{f} \right\}_{f=1}^{F}$,
it is particularly well suited to enforce multiple desired poses~${}^{f}\DataOut_{\text{des}}$ 
at given times~$\DataIn_{\text{des}, f}$.

\textbf{Proposed Extension to Local Frames:}
To enforce the desired pose~${}^{f}\DataOut_{\text{des}}$
in the $f$-th local perspective,
we identify the most suitable Gaussian component by evaluating
$
	\IndexGaussian{f}
	=
	\Argmin_{k} \left( \{ \norm{ \Mean{k}{\DataOut}{f} - {}^{f}\DataOut_{\text{des}} } \}_{k=1}^{K} \right) 
$
similar to~\eqref{eq:IdentifyClosestComponent}.
Then, means~$\Mean{\IndexGaussian{f}}{}{f}$ are fixed
with $\Mean{\IndexGaussian{f}}{\DataOut}{f} = {}^{f}_{}\DataOut_{\text{des}}$
and $\Mean{\IndexGaussian{f}}{\DataIn}{f} = \DataIn_{\text{des}, f}$
when learning local GMMs for each perspective using the CEM algorithm,
thereby ensuring $\ExpMean{\IndexGaussian{f}}{}{f} = {}^{f}_{}\DataOut_{\text{des}}$ in local frames
for the input~$\DataIn = \DataIn_{\text{des}, f}$.
	
\textbf{Proposed Fusion:}
We aim to enforce TSC in the global frame,
and therefore adapt the fusion of local predictions 
by introducing a variable weighting term~$\weighting_{f,n} \inR$ in the Product of Gaussians.
Equations~\eqref{eq:TPGMM_ManifoldProductOfGaussiansSumCovariance} and~\eqref{eq:TPGMM_ManifoldProductOfGaussiansSumMean} become
\begin{subequations}
	\begin{gather}
		\label{eq:TPGMM_ManifoldProductOfGaussiansSumCovarianceWeighted}
		\CovariancePT{n}{\Sigma}{}
		= 
		\sum_{f=1}^{F} \left( \weighting_{f,n} \,\, \CovariancePT{n}{}{f}^{-1} \right)
		\\
		\label{eq:TPGMM_ManifoldProductOfGaussiansSumMeanWeighted}
		\DataTangent{n}{\Sigma}{}
		=
		\sum_{f=1}^{F} \left( \weighting_{f,n} \,\, \CovariancePT{n}{}{f}^{-1} \logmap{\MeanGlobal{n}{}{f}}{\MeanGlobal{n}{}{}} \right)
	\end{gather}
\end{subequations}
which adapts the importance of each perspective differently.

\textbf{Proposed Weighting:}
Each local mean $\Mean{\IndexGaussian{f}}{}{f}$
enforces a desired pose ${}^{f}\DataOut_{\text{des}}$
at a desired time $\DataIn_{\text{des}, f}$.
To identify the corresponding inputs $\DataIn_{\IndexInput{f}} \in \{ \DataIn_{1} , \cdots , \DataIn_{T} \}$
for each local perspective,
we evaluate
$
	\IndexInput{f}
	=
	\Argmin_{n} ( \{ \norm{ \DataIn_{n} - \Mean{\IndexGaussian{f}}{\DataIn}{f} } \}_{n=1}^{T} ) 
$.
For each of these inputs $\DataIn_{\IndexInput{f}}$,
only one perspective
has to be responsible for generating the global prediction
to keep the desired properties,
i.e., for each $f \in \{ 1 , \cdots , F \}$ it has to hold $\weighting_{f,\IndexInput{f}} = 1$ and $\weighting_{i,\IndexInput{f}} = 0 , \forall i \neq f$.
Therefore, we choose weightings $\weighting_{f,n} \in [0,1]$ based on
\begin{equation}
	\label{eq:WeightingLinear}
	\weighting_{f,n}
	=
	\begin{cases}
		\max \left( 0, \frac{
			n-\IndexInput{f-1}
		}{
			\IndexInput{f}-\IndexInput{f-1}
		} \right) 
		& \text{if } n < \IndexInput{f}
		\\
		1
		& \text{if } n = \IndexInput{f}
		\\
		\max \left( 0, 1 - \frac{
			n-\IndexInput{f}
		}{
			\IndexInput{f+1}-\IndexInput{f}
		} \right) 
		& \text{if } n > \IndexInput{f}
	\end{cases}
\end{equation}
which utilizes a linear transition between two perspectives.

\textbf{Remark:}
If a $f$-th local perspective
for any $f \in \{ 1 , \cdots , F \}$
is not used to enforce a desired pose,
standard EM is utilized to learn its GMM.
The closest component to the perspective's origin is then identified through
$
	\IndexGaussian{f}
	=
	\Argmin_{k} \left( \{ \norm{ \Mean{k}{\DataOut}{f} } \}_{k=1}^{K} \right)
$
and
$
	\IndexInput{f}
	=
	\Mean{\IndexGaussian{f}}{\DataIn}{f}
$
is utilized to compute weightings $\weighting_{f,n}$.
This also applies to the incorporation of obstacle frames.

%% file: robot_experiments.tex
\section{Experiments}
\label{sec:Experiments}

This section evaluates the proposed approach on handwritten data and real robot experiments using a KUKA LBR iiwa robot.
For the real robot experiments,
demonstrations were recorded for the end-effector utilizing forward kinematics at 30 Hz using hand-guiding, 
and their duration was normalized through data pre-processing (in our case 60 seconds length).
Learned end-effector trajectories are executed utilizing a classical QP-controller~\cite{Bouyarmane_2019} implemented within KUKA FRI
taking into account hardware limits and manipulator dynamics.
As a baseline, we chose the GMM (resp. TP-GMM) approach with standard EM, described in Sec.~\ref{sec:Fundamentals},
and Bagging-GMM.


\subsection{Simulated Evaluation on Handwritten Data}
\label{sec:ExperimentHandwritten}

\begin{figure}
	\centering
	\includegraphics[width=0.99\linewidth]{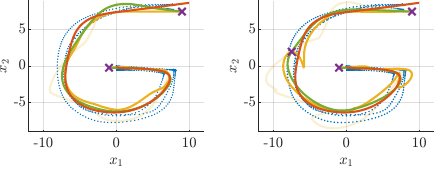}
	\caption{
		Comparison of standard EM/GMR (red), Bagging-GMM/GMR (yellow) and the proposed CEM/GMR approach (green)
		for three handwritten demonstrations (blue)
		with two (left) and three (right) desired states (purple).
	}
	\label{fig:ExperimentBaggingComparison}
\end{figure}

First, we evaluate the constrained GMM approach on 2D handwritten data,
i.e., $\DataAll{}{} = \begin{bmatrix} t , x_1 , x_2 \end{bmatrix}^T$,
from the LASA data set~\cite{Khansari_SEDS_2011}.
Results are compared to
standard EM~\cite{Ghahramani_SupervisedLearningFromIncompleteDataViaAnEmApproach_1993}
and Bagging-GMM~\cite{Ye_BaggingForGaussianMixtureRegressionInRobotLearningFromDemonstration}
in Fig.~\ref{fig:ExperimentBaggingComparison},
where we aim for two and three desired states during reproduction, respectively.
While standard EM is unconstrained,
Bagging-GMM and the proposed CEM approach allows to enforce TSC.
Each model was defined with $K = 6$ Gaussians.
For the Bagging-GMM approach,
we trained ten base learners on different subsets of the demonstration data,
which we obtained from random permutations.
Even though the desired states are reached precisely using Bagging-GMM,
the generated motion might result in large reproduction errors, 
which depends on the randomization of data subsets.
Using CEM,
on the other hand,
the constraints are enforced within the learning process,
resulting in an accurate reproduction of the demonstration data.


\subsection{Grasping Task with constrained GMM}
\label{sec:ExperimentGMM}

\begin{figure}
	\centering
	\includegraphics[width=0.99\linewidth]{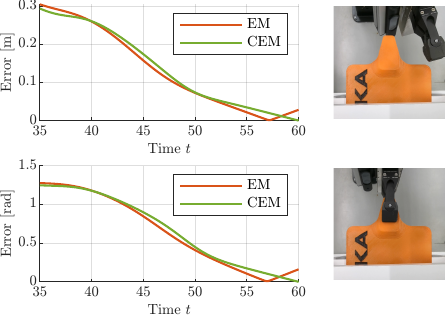}
	\caption{
		Position (top) and axis-angle (bottom) errors from GMR without (red) and with (green) TSC are shown on the left.
		Images on the right show the final state without (top) and with (bottom) enforcing TSC during learning.
	}
	\label{fig:ExperimentGrasp}
\end{figure}

Next, we evaluate the constrained GMM approach in a grasping task.
In all four demonstrations, the robot was taught to approach a chopping board from the side and grasp it at its handle.
The demonstrations include different start and target poses that encode the task variance.
As is common in related works, the data is projected onto the frame of the target pose before learning.
A GMM with $K = 5$ Gaussians is trained,
and the last component,
i.e., $\IndexGaussian{} = 5$,
enforces the TSC.
During motion execution, reference poses obtained from GMR are converted back into the global frame.

Fig.~\ref{fig:ExperimentGrasp} compares the results with the baseline.
At the end of the trajectory, the baseline approach deviates from its desired target in both position and orientation,
resulting in a state that is unable to grasp the chopping board.
Instead, the model learned with TSC generates a trajectory that stops precisely at the desired target pose,
enabling the robot to grasp the board successfully.


\begin{figure} 
	\centering
	\includegraphics[width=0.99\linewidth]{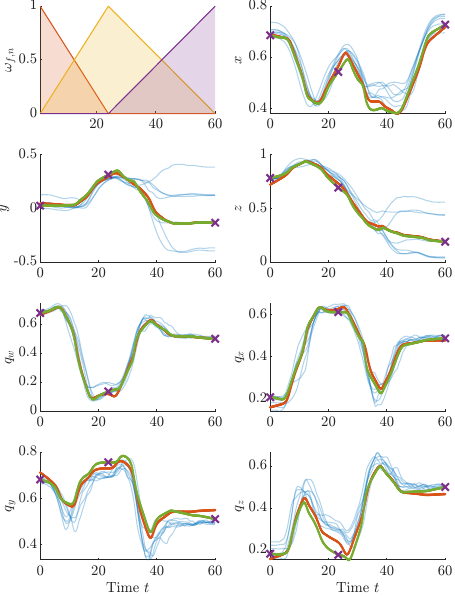}
	\caption{
		Comparison of
		enforcing TSC during learning (green) 
		and the baseline TP-GMM approach (red)
		when learning a robot end-effector motion from demonstrations (blue)
		for the inspection task.
		Desired task-parameters defined at run-time are marked in purple
		and output states are labeled as
		$\DataOut = \begin{bmatrix} x , y , z , q_w , q_x , q_y , q_z \end{bmatrix}^T$.
		Weightings $\weighting_{f,n}$ for each perspective used during fusion are visualized at the top left.
	}
	\label{fig:ExperimentInspection}
\end{figure}

\subsection{Inspection Task with constrained TP-GMM}
\label{sec:ExperimentTPGMM}

Finally, we evaluate a task-parameterized inspection task shown in Fig.~\ref{fig:InspectionTaskSetup}
that consists of
(i)~lifting a chopping board out of the rack on top of the shelf,
(ii)~presenting it to a camera, and
(iii)~inserting it into the shelf.
Seven demonstrations are provided,
and task-parameters for $F=3$ frames representing the start, target and camera inspection pose are defined.

The TP-GMM is trained with $K = 9$ Gaussians,
where components $\IndexGaussian{1} = 1$, $\IndexGaussian{2} = 4$ and $\IndexGaussian{3} = 9$
enforce the pose of the task-parameters,
i.e., the origin in each local perspective,
at given times.
Fig. \ref{fig:ExperimentInspection} compares the results with and without TSC.
As intended,
only one local perspective is responsible for generating the global prediction at the desired times.
The constrained trajectory succeeds in reaching all three desired poses when learning the model with the CEM algorithm.

%% file: conclusion.tex
\section{Conclusion}
\label{sec:Concolusion}

In this letter, we propose to enforce constraints within the Expectation-Maximization algorithm,
by fixing some of the model parameters to predetermined values in the \emph{maximization}-step.
Afterwards,
a covariance scaling is applied to enforce exclusive activation of specific components at desired time instances.
Utilizing the novel CEM algorithm enables us to guarantee time-sensitive constraints in probabilistic imitation learning.
Contrary to previous works,
these constraints are considered within the learning process,
allowing other model parameters to adapt.
Furthermore, we extended the Task-Parameterized Gaussian Mixture Model framework,
addressing one of its main challenges through CEM.
By carefully adjusting the fusion of local models, 
resulting trajectories accurately generalize to new scenarios.
Real-robot experiments with KUKA iiwa in a realistic setting validate our approach.

%% file: appendix.tex
\appendix[Convergence of the CEM Algorithm]
\label{sec:ProofCEM}

\begin{theorem}
	\label{theorem:non_decreasing_fixed_mean}
	When fixing a mean~$\Mean{k}{}{}$ within the EM algorithm,
	the likelihood remains non-decreasing in each iteration.
\end{theorem}
\begin{proof}
	In the context of learning a GMM, 
	we assume the existence of latent variables~$\{ \latent{k} \}_{k=1}^{K}$,
	which cannot be observed from the data directly,
	i.e., $\prob(\DataAll{}{}) = \sum_{k=1}^{K} \prob(\DataAll{}{}, \latent{k})$.
	The probability of $\prob(\DataAll{}{}, \latent{})$ is then factorized
	into latent probabilities~$\prob(\latent{k})$
	and conditional probabilities~$\prob(\DataAll{}{} | \latent{k})$.
	In a mixture model,
	these latent variables correspond to the assignments of a data sample to a specific Gaussian component,
	i.e., the prior probability~$\Prior{k}$,
	while the conditional probability is modeled by a Gaussian distribution~$\gauss(\DataAll{}{} | \Mean{k}{}{}, \Covariance{k}{}{})$.
	This leads to
	\begin{equation}
			\label{eq:LatentVariableModel}
			\prob(\DataAll{}{})
			=
			\sum_{k=1}^{K} \prob(\latent{k}) \prob(\DataAll{}{} | \latent{k})
			= 
			\sum_{k=1}^{K} \Prior{k} \gauss(\DataAll{}{} | \Mean{k}{}{}, \Covariance{k}{}{})
	\end{equation}

	To find model parameters~$\ModelParam{} = \{ \Prior{k} , \Mean{k}{}{} , \Covariance{k}{}{}\}_{k=1}^{K}$
	for a given set of data~$\DataSet{}$,
	the likelihood that the observed data was generated by the model is maximized.
	For this purpose,
	the log-likelihood function given in~\eqref{eq:LikelihoodFunction}
	is differentiated with respect to each model parameter.
	For a mean~$\Mean{k}{}{}$ this yields
	\begin{equation}
		\begin{aligned}
			\label{eq:MeanDerivative}
			&\frac{\partial \log \likelihood(\DataSet{} | \ModelParam{})}{\partial \Mean{k}{}{}}
			=
			\frac{\partial}{\partial \Mean{k}{}{}}
			\sum_{n=1}^{N}
			\log \sum_{k=1}^{K} \Prior{k} \gauss(\DataAll{n}{} | \Mean{k}{}{}, \Covariance{k}{}{})
			\\
			&\phantom{===}=
			\sum_{n=1}^{N}
			\frac{\partial}{\partial \Mean{k}{}{}}
			\frac{
				\sum_{k=1}^{K} \Prior{k} \gauss(\DataAll{n}{} | \Mean{k}{}{}, \Covariance{k}{}{})
			}{
				\sum_{i=1}^{K} \Prior{i} \gauss(\DataAll{n}{} | \Mean{i}{}{}, \Covariance{i}{}{})
			}
			\\
			&\phantom{===}=
			\sum_{n=1}^{N}
			\frac{
				\Prior{k} \gauss(\DataAll{n}{} | \Mean{k}{}{}, \Covariance{k}{}{})
			}{
				\sum_{i=1}^{K} \Prior{i} \gauss(\DataAll{n}{} | \Mean{i}{}{}, \Covariance{i}{}{})
			}
			(\DataAll{n}{} - \Mean{k}{}{})^{T} \Covariance{k}{}{}^{-1}
		\end{aligned}
	\end{equation}

	However, solving~\eqref{eq:MeanDerivative} for mean~$\Mean{k}{}{}$ is not possible.
	Instead, given the current estimate of model parameters $\ModelParam{}$,
	the posterior probability
	\begin{equation}
		\begin{aligned}
			\label{eq:LatentPosteriorProbability}
			\prob(\latent{k} | \DataAll{}{})
			=
			\frac{
				\prob(\latent{k}) \prob(\DataAll{}{} | \latent{k})
			}{
				\prob(\DataAll{}{})
			}
			=
			\frac{
				\Prior{k} \gauss(\DataAll{}{} | \Mean{k}{}{}, \Covariance{k}{}{})
			}{
				\sum_{i=1}^{K} \Prior{i} \gauss(\DataAll{}{} | \Mean{i}{}{}, \Covariance{i}{}{})
			}
		\end{aligned}
	\end{equation}
	is computed in the \emph{expectation} step,
	which describes that a sample~$\DataAll{}{}$ belongs to a specific cluster.
	By fixing these posterior probabilities,
	each iteration locally maximizes the likelihood
	based on the current estimate of the model parameters.
	Each updated mean~$\Mean{k}{}{}$ contributes to this increase if it was not at the local maximum already.
	Therefore, even if a mean is fixed during learning,
	the likelihood remains non-decreasing in each iteration.
\end{proof}